\documentclass[runningheads]{llncs}
\usepackage{booktabs}
\usepackage{amssymb}
\usepackage{amsmath}
\usepackage{graphicx}
\graphicspath{ {imgs/} }

\usepackage{footnote}

\usepackage[hidelinks]{hyperref}
\usepackage{cleveref}
\usepackage{footnotebackref}
\usepackage[bottom]{footmisc}

\begin{document}
\title{Decision Concept Lattice vs. Decision Trees and Random Forests}
\titlerunning{Decision Lattice}
%
\author{Egor Dudyrev\inst{1}\orcidID{0000-0002-2144-3308} \\ 
Sergei O. Kuznetsov\inst{1}\orcidID{0000-0003-3284-9001}}
\authorrunning{E. Dudyrev, S. O. Kuznetsov}
%
\institute{National Research University Higher School of Economics, Moscow, Russia}
\maketitle              
\begin{abstract}
Decision trees and their ensembles are very popular models of supervised machine learning. 
In this paper we merge the ideas underlying decision trees, their ensembles and FCA by proposing a new supervised machine learning model which can be constructed in polynomial time and is applicable for both classification and regression problems. Specifically, we first propose a polynomial-time algorithm for constructing a part of the concept lattice that is based on a decision tree. Second, we describe a prediction scheme based on a concept lattice for solving both classification and regression tasks with prediction quality comparable to that of state-of-the-art models.

\end{abstract}
\section{Introduction}

In this work we propose an approach to combining the ideas based on concept lattices and  decision trees, which are extensively used in practical machine learning (ML), in order to create a new ML model which generates good classifiers and regressors in polynomial time.

Formal Concept Analysis (FCA) is a mathematically-founded theory well suited for developing models of knowledge discovery and data mining \cite{FCA}, \cite{Genes}, \cite{MiningCare}. 
One of the serious obstacles to the broad use of FCA for knowledge discovery is that the number of formal concepts (i.e. patterns found in a data) can grow exponentially in the size of the data  \cite{KuznetsovCompare}. Sofia algorithm  \cite{Sofia} offers a solution to this problem by constructing only a limited amount of most stable concepts.

Learning decision trees (DT) \cite{DT} is one of the most popular supervised machine learning approaches. 
Most famous methods based on ensembles of decision trees -- aimed at increasing the accuracy of a single tree -- are random forest (RF) \cite{RF} and gradient boosting over decision trees \cite{GB}. 
Both algorithms are considered among the best in terms of accuracy \cite{Catboost}.

There are a number of papers which highlight the connection between the concept lattice and the decision tree. The work \cite{InducingDT} states that a decision tree can be induced from a concept lattice. In \cite{KuznetsovDT} the author compares the ways the concept lattice and the decision tree can be used for supervised learning. Finally, in \cite{Krause2020ALB} the authors  provide a  deep mathematical explanation on the connection between the concept lattice and the decision tree.

In this paper we develop the previous work in a more practical way. We show that 
the decision tree (and its ensembles) can 
induce a subset of concepts of the concept lattice. We propose a polynomial-time algorithm to construct a supervised machine learning model based on a concept lattice with prediction 
quality comparable to that of the state-of-the-art models.

\section{Basic Definitions}

For standard definitions of FCA and decision trees we refer the reader to~\cite{gw99} and~\cite{DT}, respectively.

In what follows we describe algorithms for binary attributes, numerical data can be processed by means of interval pattern structures or can be scaled to binary contexts~\cite{PS_for_complex}.

\section{Construct a Concept Lattice via a set of Decision Trees}

\begin{definition}[Classification rule]
Let $M$ be a set of attributes of a context $\mathbb{K}$ and $Y$ be a set of ``target'' values. A pair $(\rho, \hat{y}_\rho), \rho \subseteq M, \hat{y}_\rho \in Y$ is a classification rule where $\rho$ is called a premise and $\hat{y}_\rho$ is a target prediction.
\end{definition}

Applied to object $g \subseteq G$ it can be interpreted as ``if the description of $g$ falls under the premise $\rho$, then object $g$ should have the target value $\hat{y}_\rho$'' or ``$\text{if } \rho \subseteq g' \Rightarrow \hat{y}_\rho$''.

In the case of classification task $Y$ can be represented either as a set $\{0, 1\}$: $Y = \{ y \in \{0, 1\}\}_{i=1}^{|G|}$ or a set of probabilities of a positive class: $Y = \{y \in [0,1]\}_{i=1}^{|G|}$. In the case of regression task target value $Y$ is a set of real valued numbers: $Y = \{y \in \mathbb{R}\}_{i=1}^{|G|}$.

We can define a decision tree $DT$ as a partially ordered set (poset) of classification rules:

\begin{equation}
     DT \subseteq \{ (\rho, \hat{y}_{\rho}) \mid \rho \subseteq M, \hat{y}_\rho \in Y \}
\end{equation}
where by \emph{the order of classification rules} we mean the inclusion order on their premises:
\begin{equation}
    (\rho_1, \hat{y}_{\rho_1}) \leq (\rho_2, \hat{y}_{\rho_2}) \Leftrightarrow \rho_1 \subseteq \rho_2 
\end{equation}

Here we assume that a decision tree is a binary tree, i.e. its node is either a leaf (has no children) or has exactly 2 children nodes.

The other property of a decision tree is that each premise of its classification rules describes its own unique subset of objects:
\begin{equation}
\label{eq:unique_clfrule}
    \forall (\rho_1, \hat{y}_{\rho_1}) \in DT, \nexists (\rho_2, \hat{y}_{\rho_2}) \in DT: \rho_1' = \rho_2'
\end{equation}

These simple properties result in an idea that 1) we can construct a concept lattice by closing premises of a decision tree, 2) join semilattice of such concept lattice is isomorphic to a decision tree.

\begin{proposition}
\label{prop:lattice_by_dt}
Let $\mathbb{K}=(G,M,I)$ be a formal context, $L(\mathbb{K})$ be a lattice of the context $\mathbb{K}$. A subset of formal concepts $L_{DT}(\mathbb{K})$ 
forming a lattice can be derived from the decision tree $DT(\mathbb{K})$ constructed from the same context as:
\begin{equation}
    L_{DT} = \{ (\rho', \rho'') \mid \forall (\rho', \hat{y}_\rho ) \in DT(\mathbb{K}) \} \cup \{ (M', M) \}
\end{equation}
\end{proposition}

\begin{proposition}
\label{prop:dt_semlattice}
Join-semilattice of a concept lattice $L_{DT}$ is isomorphic to the decision tree $DT$.
\end{proposition}
\begin{proof}
Given two classification rules $(\rho_1, \hat{y}_{\rho_1}),(\rho_1, \hat{y}_{\rho_1}) \in DT$ let us consider two cases:
\begin{enumerate}
    \item $\rho_1 \subseteq \rho_2 \Rightarrow (\rho_1', \rho_1'') \leq (\rho_2', \rho_2'')$
    \item $\rho_1 \not \subseteq \rho_2, \rho_2 \not \subseteq \rho_1 \Rightarrow \exists m \in M: m \in \rho_1, \neg m \in \rho_2 \\ \Rightarrow (\rho_1', \rho_1'') \not \leq (\rho_2', \rho_2''), (\rho_2', \rho_2'') \not \leq (\rho_1', \rho_1'')$
\end{enumerate}

Thus the formal concepts from the join-semilattice of $L_{DT}$ possess the same partial order as the classification rules from $DT$.
\end{proof}

Since we can construct a concept lattice from a decision tree and there is a union operation for concept lattices then we can construct a concept lattice which will correspond to a ``union'' of a number of independent decision trees (i.e. a random forest). 

\begin{proposition}
\label{prop:lattice_by_rf}
Let $\mathbb{K}=(G,M,I)$ be a formal context, $L(\mathbb{K})$ be a lattice of the context $\mathbb{K}$. A subset of formal concepts $L_{RF}(\mathbb{K})$ of the concept lattice $L(\mathbb{K})$
forming a lattice can be obtained from a random forest, i.e. from a set of $m$ decision trees constructed on subsets of a formal context $DT_i(K_i), i=1,...,m, K_i \subseteq \mathbb{K}$:
\begin{equation}
    L_{RF}(\mathbb{K}) = \bigcup_{i=1}^m L_{DT_i}(K_i)
\end{equation}
\end{proposition}

The size of the lattice $L_{RF}$ is close to the size of the underlying random forest $RF$:  $|L_{RF}| \sim |RF| \sim O(mG\log(G))$, where $m$ is the number of trees in $RF$~\cite{SklearnRF}.
According to~\cite{SklearnDT} the time complexity of  constructing a decision tree is $O(MG^2\log(G))$. Several algorithms for constructing decision trees and random forests are implemented in various libraries and frameworks like Sci-kit learn\footnote{\url{https://scikit-learn.org/stable/modules/ensemble.html\#random-forests}}
, H2O\footnote{\url{http://h2o-release.s3.amazonaws.com/h2o/master/1752/docs-website/datascience/rf.html}}
, Rapids\footnote{\url{https://docs.rapids.ai/api/cuml/stable/api.html\#random-forest}}
. The latter is even adapted to be run on GPU.

Thus, our lattice construction algorithm has two steps:
\begin{enumerate}
    \item Construct a random forest $RF$
    \item Use random forest $RF$ to construct a concept lattice $L_{RF}$ (by eq. \ref{prop:lattice_by_rf})
\end{enumerate}

Both strong and weak side of this algorithm is that it relies on a supervised machine learning model, so it can be applied only if target  labels $Y$ are given. In addition, the result set of concepts may not be optimal w.r.t. any concept interestingness measure \cite{km2018}. Though it is natural to suppose that such set of concepts should be reasonable for supervised machine learning tasks.

\section{Decision Lattice}

Given a formal concept $(A, B)$ we can use its intent $B$ as a premise of a classification rule $(B, \hat{y}_B)$.

The target prediction $\hat{y}_B$ of such classification rule $(B, \hat{y}_B)$ can be estimated via an aggregation function over the set $\{y_g \mid \forall g \in A\}$. In what follows we use the average aggregation function:
\begin{equation}
    \hat{y}_B = \frac{1}{|A|} \sum_{\forall g \in A} y_g
    \label{eq:clf_rule}
\end{equation}

Let us define a decision lattice (DL) as a poset of classification rules.

\begin{definition}
Let $M$ be a set of attributes of a formal context $\mathbb{K}$ and $Y$ be a set of target values. Then a poset of classification rules is called a decision lattice $DL$ if a premise of each classification rule of $DL$ describes its own unique subset of objects (similar to $DT$ in equation \ref{eq:unique_clfrule})
\end{definition}

Decision lattice $DL$ can be constructed from a concept lattice $L$ as follows:
\begin{equation}
    DL = \{ (B, \hat{y}_B) \mid (A, B) \in L \}
\end{equation}
where $\hat{y}_B$ can be computed in various ways (we use the equation \ref{eq:clf_rule}).

To get a final prediction $\hat{y}_g$ for an object $g$ a decision tree $DT$ firstly selects all the classification rules $DT^g$ describing the object $g$. Then it uses the target prediction of the maximal classification rule from $DT^g$
\begin{align}
    DT^g &= \{(\rho, \hat{y}_\rho) \in DT \mid \rho \subseteq g'\} \\
    DT^g_{max} &= \{(\rho, \hat{y}_\rho) \in DT^g \mid \nexists (\rho_1, \hat{y}_{\rho_1}) \in DT^g: \rho \subset \rho_1\} \\
    \hat{y}_g &= \hat{y}_\rho, \quad (\rho, \hat{y}_\rho) \in DT^g_{max}
\end{align}

We use the same algorithm to get a final prediction $\hat{y}_g$ for an object $g$ by a decision lattice $DL$. The only difference is that when the subset $DT^g_{max}$ always contains only one classification rule a subset $DL^g_{max}$ may contain many. In this case we average the predictions of maximal classification rules $DL^g_{max}$:
\begin{align}
    \hat{y}_g &= \frac{1}{|DL^g_{max}|} \sum_{(\rho, \hat{y}_\rho) \in DL^g_{max}} \hat{y}_\rho
\end{align}

\begin{table}[]
    \centering
    \tabcolsep=0.06cm
     \fontsize{6}{8}\selectfont
\begin{tabular}{l||lllllllll||l}
{} &   \multicolumn{9}{l}{attributes M} & label Y \\
\hline
{} &   firm & smooth & \multicolumn{4}{l}{color} & \multicolumn{3}{l}{form} &  fruit \\
{objects G} &     &     & yellow &  green &   blue &  white &  round &   oval &  cubic &     \\
\hline
\hline
apple       &   &   X &   X &   &   &   &   X &   &   &   1 \\
grapefruit  &   &   &   X &   &   &   &   X &   &   &   1 \\
kiwi        &   &   &   &   X &   &   &   &   X &   &   1 \\
plum        &   &   X &   &   &   X &   &   &   X &   &   1 \\
toy cube    &   X &   X &   &   X &   &   &   &   &   X & 0  \\
egg         &   X &   X &   &   &   &   X &   &   X &   & 0  \\
tennis ball &   &   &   &   &   &   X &   X &   &   &  0 \\
\hline
mango       &   &   X &   &   X &   &   &   &   X &   &   1 \\
\end{tabular}
    \caption{\emph{Fruit context} and \emph{fruit} labels}
    \label{tab:fruit_ctx}
\end{table}

Let us consider the \emph{fruit context} $\mathbb{K}=(G,M,I)$ and \emph{fruit} label $Y$ presented in Table \ref{tab:fruit_ctx}.
We want to compare the way decision lattice makes an estimation of the label $y_{mango}$ of object \emph{mango} when this object is included in the train or the test context.

Figure \ref{fig_mangolattice} represents decision lattices constructed upon \emph{fruit context} with (on the left) and without (in the center) \emph{mango} object. In both cases we show only the classification rules which cover (describe) \emph{mango} object. 

The left picture represents a decision lattice with 8 classification rules and 1 single maximal classification rule: (``color\_is\_green \& form\_is\_oval \& smooth'', 1). Therefore we use this classification rule to predict the target label of \emph{mango}.

The picture in the center shows a decision lattice with 6 classification rules and 2 maximal classification rules: (``color\_is\_green \& form\_is\_oval'', 1),  \\(``form\_is\_oval \& smooth'', 1/2). We average the target predictions of these classification rules to get a final prediction of $3/4$ as shown in the picture on the right.

\begin{figure}[]
\includegraphics[width=\textwidth, height=40mm]{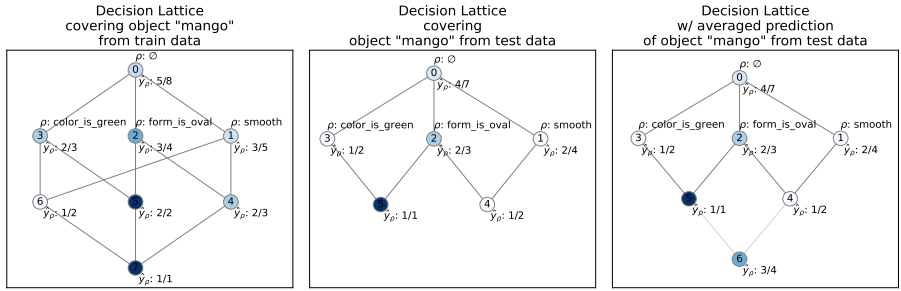}
\caption{Example of prediction of mango object} \label{fig_mangolattice}
\end{figure}

\section{Experiments}

We compare our decision lattice (DL) approach with the most popular 
machine learning models on real world datasets. One can reproduce the results by running Jupyter notebooks stored on GitHub \cite{dl_eval}. Decision lattice models are implemented in open-source Python library for FCA which is called FCApy and located in the same GitHub repository.

We use ensembles of 5 and 10 decision trees to construct decision lattice models DL\_RF\_5 and DL\_RF\_10, respectively. The same ensemble of 5 decision trees is used by random forest models RF\_5. Thus, we can compare prediction qualities of DL\_RF\_5 and RF\_5 based on the same set of decision trees (and, consequently, classification rules).

The non-FCA model we use for the comparison are decision tree (DT), random forest (RF) and gradient boosting (GB) from sci-kit learn library, gradient boostings from LightGBM (LGBM), XGBoost (XGB), CatBoost (CB) libraries. 

We also test Sofia algorithm \cite{Sofia} as a polynomial-time approach to construct a decision lattice DL\_Sofia. 
We compute only 100 of most stable concepts by Sofia algorithm because of its time inefficiency.

Metadata of the datasets is given in Table \ref{tab:dataset_desc}.

\begin{savenotes}
\begin{table}[]
\parbox{.25\linewidth}{
    \centering
    \tabcolsep=0.06cm
     \fontsize{6}{8}\selectfont
\begin{tabular}{l|rrrr}
Dataset name & Task type &  \# Instances & \# Attrsibutes\\
\hline
adult \footnote{\url{https://archive.ics.uci.edu/ml/datasets/Adult}} & Bin. class. & 48842 & 14\\
amazon \footnote{\url{https://www.kaggle.com/c/amazon-employee-access-challenge/data}} & Bin. class. & 32770 & 10\\
bank \footnote{\url{https://archive.ics.uci.edu/ml/datasets/bank+marketing}} & Bin. class. & 45211 & 17 \\
breast \footnote{\url{https://archive.ics.uci.edu/ml/datasets/Breast+Cancer+Wisconsin+(Diagnostic)}} & Bin. class. & 569  & 32 \\
heart \footnote{\url{https://archive.ics.uci.edu/ml/datasets/heart+Disease}} & Bin. class. & 303 & 75 \\
kick \footnote{\url{https://www.kaggle.com/c/DontGetKicked/data?select=training.csv}} & Bin. class. & 72984 & 34\\
mammographic \footnote{\url{http://archive.ics.uci.edu/ml/datasets/mammographic+mass}} & Bin. class. & 961 & 6 \\
seismic \footnote{\url{https://archive.ics.uci.edu/ml/datasets/seismic-bumps}} & Bin. class. & 2584 & 19\\
\hline
boston \footnote{\url{https://archive.ics.uci.edu/ml/machine-learning-databases/housing}} & Regression & 506  & 14\\
calhouse \footnote{\url{https://scikit-learn.org/stable/datasets/real_world.html\#california-housing-dataset}} & Regression & 20640 & 8\\
diabetes \footnote{\url{https://scikit-learn.org/stable/datasets/toy_dataset.html\#diabetes-dataset}} & Regression & 442  & 10 \\

\end{tabular}
\caption{Description of the datasets}
    \label{tab:dataset_desc}
}
\hfill
\parbox{.45\linewidth}{

    \centering
    \tabcolsep=0.06cm
     \fontsize{6}{8}\selectfont
\begin{tabular}{l||rr|rr|rr||rr}
\toprule
{} & \multicolumn{2}{l}{boston} & \multicolumn{2}{l}{calhouse} & \multicolumn{2}{l}{diabetes} & \multicolumn{2}{l}{mean delta} \\
{} &  train &  test &    train &  test &    train &  test &     train &  test \\
model            &        &       &          &       &          &       &           &       \\
\hline
\hline
DL\_RF\_5          &   0.02 &  0.06 &     0.14 &  0.05 &     0.05 &  0.00 &      0.07 &  0.04 \\
\hline
DL\_RF\_10         &   0.01 &  0.07 &     0.13 &  0.04 &     0.01 &  0.01 &      0.05 &  0.04 \\
\hline
DL\_Sofia         &   0.29 &  0.20 &       &    &     0.40 &  0.11 &      0.35 &  0.16 \\
\hline
DT               &   0.00 &  0.05 &     0.00 &  0.09 &     0.00 &  0.12 &      0.00 &  0.09 \\
RF\_5             &   0.05 &  0.01 &     0.14 &  0.04 &     0.15 &  0.02 &      0.12 &  0.03 \\
RF               &   0.04 &  0.00 &     0.06 &  0.02 &     0.12 &  0.00 &      0.07 &  0.01 \\
GB               &   0.05 &  0.00 &     0.17 &  0.02 &     0.16 &  0.00 &      0.13 &  0.01 \\
LGBM             &   0.04 &  0.01 &     0.13 &  0.00 &     0.11 &  0.00 &      0.09 &  0.01 \\
CB               &   0.02 &  0.00 &     0.13 &  0.00 &     0.06 &  0.00 &      0.07 &  0.00 \\
\hline
\hline
best result      &   0.00 &  0.14 &     0.00 &  0.21 &     0.00 &  0.31 &      0.00 &  0.22 \\
\bottomrule
\end{tabular}
    \caption{Weighted Average Percentage Error (best model delta)}
    \label{tab:wape_delta}
}

\end{table}
\end{savenotes}

For each dataset we use 5-fold cross-validation. We compute F1-score to measure the predictive quality of classification and weighted average percentage error (WAPE) to that of regression. In Tables \ref{tab:f1_delta}--\ref{tab:wape_delta} we show the difference between the metric value of the model and the best obtained metric value among all methods.

As can be seen from Tables \ref{tab:f1_delta}--\ref{tab:wape_delta} DL\_RF model does not always show the best result among all the tested models, though its prediction quality is comparable to the state-of-the-art.

DL\_Sofia model shows the worst results. There may be 2 reasons for this. First, it uses only a hundred of concepts. Second, we use Sofia algorithm to find one of the most stable concepts, but not the ones which minimize the loss.

Figure \ref{fig_time} shows the time needed to construct a lattice by the sets of 5 ($DL\_RF\_5$) and 10 ($DL\_RF\_10$) decision trees and by Sofia algorithm ($DL\_Sofia$).
The lattice can be constructed in a time linear in the number of objects in the given data. 

\begin{table}[]
    \centering
    \tabcolsep=0.06cm
     \fontsize{6}{8}\selectfont
\begin{tabular}{l||rr|rr|rr|rr|rr|rr|rr|rr||rr}
\toprule
{} & \multicolumn{2}{l}{adult} & \multicolumn{2}{l}{amazon} & \multicolumn{2}{l}{bank} & \multicolumn{2}{l}{breast} & \multicolumn{2}{l}{heart} & \multicolumn{2}{l}{kick} & \multicolumn{2}{l}{mamm.} & \multicolumn{2}{l}{seismic} & \multicolumn{2}{l}{mean delta} \\
{} & train &  test &  train &  test & train &  test &  train &  test & train &  test & train &  test &      train &  test &   train &  test &     train &  test \\
model            &       &       &        &       &       &       &        &       &       &       &       &       &            &       &         &       &           &       \\
\hline
\hline
DL\_RF\_5          & -0.35 & -0.06 &  -0.01 & -0.00 & -0.41 & -0.16 &  -0.01 & -0.01 & -0.03 & -0.02 & -0.59 & -0.03 &      -0.03 & -0.02 &   -0.24 & -0.15 &     -0.21 & -0.05 \\
DL\_RF\_10         & -0.33 & -0.05 &  -0.01 & -0.00 & -0.37 & -0.14 &  -0.00 & -0.00 & -0.01 &  0.00 & -0.58 & -0.03 &      -0.01 & -0.02 &   -0.13 & -0.15 &     -0.18 & -0.05 \\
DL\_Sofia         &    &    &     &    &    &    &  -1.00 & -0.95 & -0.33 & -0.27 &    &    &      -0.87 & -0.72 &   -1.00 & -0.15 &     -0.80 & -0.52 \\
\hline
DT               &  0.00 & -0.10 &  -0.00 & -0.01 &  0.00 & -0.24 &   0.00 & -0.06 &  0.00 & -0.15 &  0.00 & -0.05 &      -0.00 & -0.08 &    0.00 &  0.00 &     -0.00 & -0.09 \\
RF\_5             & -0.35 & -0.05 &  -0.01 & -0.00 & -0.41 & -0.12 &  -0.01 & -0.01 & -0.05 & -0.07 & -0.60 & -0.02 &      -0.04 & -0.02 &   -0.36 & -0.07 &     -0.23 & -0.05 \\
RF               & -0.00 & -0.04 &   0.00 &  0.00 &  0.00 & -0.11 &   0.00 &  0.00 &  0.00 & -0.00 & -0.00 & -0.01 &       0.00 & -0.03 &   -0.00 & -0.12 &     -0.00 & -0.04 \\
GB               & -0.36 & -0.02 &  -0.01 & -0.00 & -0.47 &  0.00 &   0.00 & -0.01 & -0.07 & -0.01 & -0.62 & -0.00 &      -0.07 &  0.00 &   -0.42 & -0.09 &     -0.25 & -0.02 \\
LGBM             & -0.31 & -0.00 &  -0.01 & -0.00 & -0.32 & -0.04 &   0.00 & -0.02 &  0.00 & -0.03 & -0.60 & -0.00 &      -0.03 & -0.02 &   -0.00 & -0.11 &     -0.16 & -0.03 \\
CB               & -0.31 &  0.00 &  -0.01 & -0.00 & -0.31 & -0.05 &   0.00 & -0.01 & -0.01 & -0.01 & -0.59 &  0.00 &      -0.04 & -0.01 &   -0.33 & -0.13 &     -0.20 & -0.02 \\
\hline
\hline
best result         &  1.00 &  0.65 &   0.98 &  0.97 &  1.00 &  0.48 &   1.00 &  0.95 &  1.00 &  0.76 &  1.00 &  0.35 &       0.95 &  0.81 &    1.00 &  0.15 &      0.99 &  0.64 \\
\bottomrule
\end{tabular}
    \caption{F1 score (best model delta)}
    \label{tab:f1_delta}
\end{table}

\begin{figure}[]
\center{\includegraphics[width=0.7\textwidth, height=50mm ]{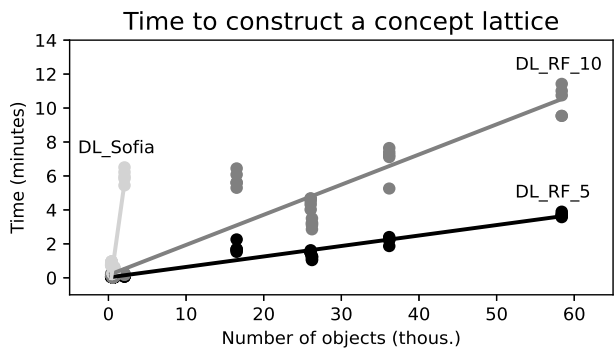}}
\caption{Time needed to construct a lattice} \label{fig_time}
\end{figure}


\section{Conclusions}
In this paper we have introduced a new concept-based method to classification and regression. The proposed method constructs concept-based classifiers obtained with decision trees and random forests.  This method is quite efficient and can be used for big datasets. We have shown that our approach is non-inferior to the  predictive quality of the state-of-the-art competitors.

In the future work we plan to extend the algorithm for constructing decision trees in the case of data given by pattern structures. 

\section*{Acknowledgments}

The work of Sergei O. Kuznetsov on the paper was carried out
at St. Petersburg Department of Steklov Mathematical Institute of Russian Academy of
Science and supported by the Russian Science Foundation grant no. 17-11-01276

%
%
%
\bibliographystyle{splncs04}
\bibliography{bibliography}

\end{document}